\documentclass{article}

\usepackage{microtype}
\usepackage{graphicx}
\usepackage{subcaption}
\usepackage{booktabs}
\usepackage{placeins}
\usepackage{multirow}
\usepackage{hyperref}

\usepackage[accepted]{icml2026}

\usepackage{amsmath}
\usepackage{amssymb}
\usepackage{mathtools}
\usepackage{amsthm}
\usepackage{xspace}

\usepackage[capitalize,noabbrev]{cleveref}

\usepackage{tikz}
\usepackage{pgfplots}
\pgfplotsset{compat=1.18}
\usetikzlibrary{arrows.meta, positioning, shapes.geometric, calc, fit,
                 decorations.pathreplacing, backgrounds}

\usepackage{algorithm}
\usepackage{algorithmic}

\definecolor{rkscblue}{HTML}{0072B2}
\definecolor{rkscgreen}{HTML}{009E73}
\definecolor{rkscorange}{HTML}{E69F00}
\definecolor{rkscpink}{HTML}{CC79A7}
\definecolor{rkscgray}{HTML}{999999}
\definecolor{rksccyan}{HTML}{56B4E9}
\definecolor{rkscred}{HTML}{D55E00}

\theoremstyle{plain}
\newtheorem{hheorem}{Theorem}[section]
\newtheorem{proposition}[hheorem]{Proposition}

\newtheorem{corollary}[hheorem]{Corollary}
\theoremstyle{definition}

\theoremstyle{remark}

\newcommand{\ie}{i.e.,\xspace}
\newcommand{\eg}{e.g.,\xspace}

\icmltitlerunning{RKSC: Reasoning-Aware KV Cache Sharing and Confident Early Exit}

\begin{document}

\twocolumn[
  \icmltitle{RKSC: Reasoning-Aware KV Cache Sharing \\
and Confident Early Exit for Multi-Step LLM Inference}

  \icmlsetsymbol{equal}{*}

  \begin{icmlauthorlist}
    \icmlauthor{Anirudh Sekar}{}
  \end{icmlauthorlist}

  \icmlcorrespondingauthor{Anirudh Sekar}{anirudhsekar2008@gmail.com}

  \icmlkeywords{Machine Learning, ICML, KV Cache, Inference Optimization, Reasoning}

  \vskip 0.3in
]

\printAffiliationsAndNotice{}

\begin{abstract}
We introduce \textbf{RKSC} (Reasoning-Aware KV Cache Sharing), a training-free inference framework that eliminates two structural redundancies in multi-branch LLM reasoning pipelines. \textbf{ASKS} (Attention-Similarity KV Sharing) computes the prefix KV cache once and broadcasts it to all semantically similar branches via hidden-state cosine similarity, strictly generalising the token-exact prefix caching used by vLLM and SGLang. \textbf{CGEE} (Confidence-Gated Early Exit) applies two complementary exit mechanisms: (1) it skips the verification forward pass entirely when generation confidence is decisive across branches, and (2) it terminates the verification pass at an intermediate layer when per-layer entropy stabilises, using lightweight hooks on the transformer backbone. \textbf{RSBCM} (Reasoning-Selective Block Cache Manager) prevents unbounded cache growth via attention-weighted depth-priority eviction. Across five model families (7B--10B), four benchmarks, and 1,000 evaluated problems, RKSC achieves a mean speedup of $\mathbf{3.008\times}$ over the No-KV baseline (peak $\mathbf{3.990\times}$), a $\mathbf{1.66\times}$ mean improvement over vLLM-equivalent prefix caching, with a CGEE-induced error rate of only $0.37\%$ (6 errors out of 1,616 verify calls). No fine-tuning or architecture changes are required. Code is available at \url{https://github.com/AnirudhSekar/RKSC}.
\end{abstract}

\section{Introduction}

Inference-time scaling systems, DeepSeek-R1 \cite{Guo_2025}, Qwen3 \cite{yang2025qwen3technicalreport}, o1-style models, generate multiple reasoning trajectories per problem, with the first trajectory correct in the majority of cases \cite{chen2025think23overthinkingo1like}. This multi-branch paradigm incurs two structural redundancies that no production system addresses. \textbf{(1) Cross-branch KV waste:} parallel branches share the majority of their prefix KV computation yet recompute it independently, in a Tree of Thoughts \cite{yao2023treethoughtsdeliberateproblem} run (branching~4, depth~3), branches diverging at depth~2 share $66\%$ of KV work by construction. vLLM \cite{kwon2023efficientmemorymanagementlarge} and SGLang \cite{zheng2024sglangefficientexecutionstructured} reuse KV blocks on exact token matches only, leaving all cross-branch similarity unexploited. \textbf{(2) Over-deep verification:} process-reward verification \cite{lightman2023letsverifystepstep} runs full-depth forward passes even when the model is already confident; no system exploits per-layer entropy collapse to exit early \emph{within} a single verify pass.

We introduce \textbf{RKSC} (Reasoning-Aware KV Cache Sharing), a training-free inference framework eliminating both redundancies. Our contributions are:

\begin{itemize}
    \item \textbf{ASKS} (\S\ref{sec:method:kv}): gates prefix KV reuse across branches via \emph{hidden-state cosine similarity} with exponentially weighted later-layer emphasis, strictly generalising token-exact caching. Achieves $28.6\%$ additional reuse where token-exact matching yields $0\%$.

    \item \textbf{CGEE} (\S\ref{sec:method:cgee}): dual-level exit combining (a)~\emph{verify-skip} when generation confidence is decisive, and (b)~\emph{layer-level entropy exit} terminating the verify pass when per-layer logit entropy stabilises.

\end{itemize}

KV caching \cite{vaswani2023attentionneed} and block-level prefix reuse (vLLM \cite{kwon2023efficientmemorymanagementlarge}, SGLang \cite{zheng2024sglangefficientexecutionstructured}) require byte-identical token prefixes; ASKS decouples reuse from lexical identity. MemShare \cite{chen2025memsharememoryefficientinference} deduplicates KV states \emph{within} a single chain via step-delimiter segmentation; ASKS operates \emph{across} parallel branches. Layer-wise early exit has been applied to BERT \cite{xin2020deebertdynamicearlyexiting} and generation \cite{schuster2022confidentadaptivelanguagemodeling,Elhoushi_2024}; for reasoning, \cite{yang2025dynamicearlyexitreasoning} exit \emph{between} chains at transition tokens, a token-level mechanism orthogonal to CGEE's within-pass layer-level exit. Self-consistency \cite{wang2023selfconsistencyimproveschainthought} and process-reward methods \cite{lightman2023letsverifystepstep} motivate multi-branch generation; RKSC accelerates any such pipeline as a drop-in wrapper, without altering the reasoning algorithm.

A complementary line of inference acceleration uses a small draft model to propose candidate tokens, which a larger target model then verifies in parallel. \cite{leviathan2023fastinferencetransformersspeculative} establish that this approach preserves the output distribution exactly while achieving 2–3× wall-clock speedup, and SpecInfer \cite{Miao_2024}  generalises it to tree-structured candidate sets verified in a single batched forward pass. Medusa \cite{cai2024medusasimplellminference} removes the need for a separate draft model entirely by appending lightweight decoding heads to the backbone, predicting multiple future tokens in parallel and verifying them via tree attention. RKSC's CGEE is complementary to all three: whereas speculative methods accelerate the decode phase by parallelising token proposals, CGEE accelerates the verification phase of process-reward scoring by gating or truncating the verify forward pass based on per-layer entropy collapse. 
\section{RKSC Pipeline}
\label{sec:method}

RKSC accelerates multi-branch reasoning inference through three complementary mechanisms: (1)~\textbf{KV prefix sharing} eliminates redundant prefill computation across branches, (2)~\textbf{CGEE} reduces or eliminates the verification forward pass, and (3)~\textbf{RSBCM} manages cache capacity under deep tree search. Figure~\ref{fig:rksc_overview} illustrates the full pipeline. The design exploits structure that already exists in the inference workload, shared prefixes, concentrated generation confidence, entropy collapse in later layers, rather than introducing new computation. Each mechanism is independently switchable and adds zero learnable parameters.
% EDIT: "The design philosophy is to exploit" -> "The design exploits" (action in verb, per Gopen & Swan)

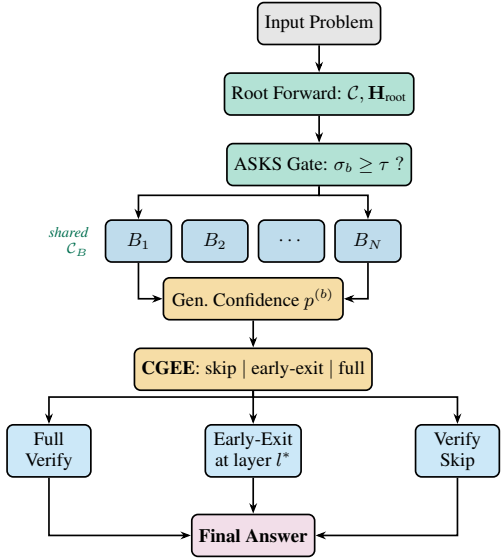
\begin{figure}[h]
\centering
\resizebox{0.8\linewidth}{!}{%
\begin{tikzpicture}[
    node distance=0.45cm and 0.3cm,
    every node/.style={font=\small},
    box/.style={draw, rounded corners=3pt, minimum height=0.7cm,
                minimum width=1.7cm, align=center, thick},
    prefill/.style={box, fill=rkscgray!25},
    asks/.style={box, fill=rkscgreen!30},
    branch/.style={box, fill=rkscblue!25, minimum width=1.1cm},
    cgee/.style={box, fill=rkscorange!35},
    verify/.style={box, fill=rksccyan!30},
    final/.style={box, fill=rkscpink!25},
    arr/.style={-{Stealth[length=5pt]}, thick},
]
\node[prefill] (input) {Input Problem};
\node[asks, below=of input] (root) {Root Forward: $\mathcal{C}, \mathbf{H}_{\text{root}}$};
\draw[arr] (input) -- (root);
\node[asks, below=of root, minimum width=3cm] (asks_gate) {ASKS Gate: $\sigma_b \geq \tau\;?$};
\draw[arr] (root) -- (asks_gate);
\node[branch, below left=0.55cm and 0.9cm of asks_gate] (b1) {$B_1$};
\node[branch, right=0.15cm of b1] (b2) {$B_2$};
\node[branch, right=0.15cm of b2] (b3) {$\cdots$};
\node[branch, right=0.15cm of b3] (b4) {$B_N$};
\draw[arr] (asks_gate.south) -- ++(0,-0.15) -| (b1.north);
\draw[arr] (asks_gate.south) -- ++(0,-0.15) -| (b4.north);
\node[left=0.15cm of b1, font=\scriptsize\itshape, text=rkscgreen!70!black, align=right] {shared\\$\mathcal{C}_B$};
\node[cgee, below=0.6cm of $(b2)!0.5!(b3)$, minimum width=3cm] (conf) {Gen.\ Confidence $p^{(b)}$};
\draw[arr] (b1.south) |- (conf.west);
\draw[arr] (b4.south) |- (conf.east);
\node[cgee, below=of conf, minimum width=3cm] (cgee) {\textbf{CGEE}: skip $|$ early-exit $|$ full};
\draw[arr] (conf) -- (cgee);
\node[verify, below left=0.55cm and 0.7cm of cgee, minimum width=1.4cm] (full) {Full\\Verify};
\node[verify, below=0.55cm of cgee, minimum width=1.6cm] (early) {Early-Exit\\at layer $l^*$};
\node[verify, below right=0.55cm and 0.7cm of cgee, minimum width=1.4cm] (skip) {Verify\\Skip};
\draw[arr] (cgee.south) -- ++(0,-0.1) -| (full.north);
\draw[arr] (cgee.south) -- (early.north);
\draw[arr] (cgee.south) -- ++(0,-0.1) -| (skip.north);
\node[final, below=0.6cm of early] (answer) {\textbf{Final Answer}};
\draw[arr] (full.south) |- (answer.west);
\draw[arr] (early.south) -- (answer.north);
\draw[arr] (skip.south) |- (answer.east);
\end{tikzpicture}
}%
\caption{\textbf{RKSC pipeline overview.} A single root forward computes the shared KV cache $\mathcal{C}$ and root hidden states. ASKS broadcasts $\mathcal{C}$ only to branches with similarity $\sigma_b \geq \tau$. All $B$ branches decode in one batched call while accumulating generation confidence. CGEE selects among three paths: \emph{verify skip} (confidence decisive), \emph{early-exit verify} (entropy stabilised at $l^*$), or \emph{full verify}.}
\label{fig:rksc_overview}
\end{figure}

\subsection{Problem Setting}
\label{sec:method:setting}

We consider an LLM $f_\theta$ with $L$ transformer layers solving a problem~$x$ by generating $B$ reasoning branches from a shared problem prefix. Each branch $b \in \{1,\ldots,B\}$ begins from a shared context $\mathbf{c} = \texttt{prefix}(x)$ of $n$ tokens and appends a short, branch-specific suffix (a reasoning hint of $s \ll n$ tokens) before decoding $t$ answer tokens autoregressively. A \emph{verify} pass then scores all $B$ branches and selects the best answer. The standard inference procedure, which we call the \emph{No-KV} baseline, batches all $B$ branches in a single forward pass but recomputes the full $n$-token prefix attention independently for each branch, paying $O(Bn^2)$ attention cost. A stronger baseline, the \emph{vLLM-equivalent}, performs token-exact prefix matching: the root KV cache is reused only when branch tokens share a byte-identical prefix with the root, mirroring the PagedAttention prefix-cache path in vLLM~\cite{kwon2023efficientmemorymanagementlarge} and SGLang~\cite{zheng2024sglangefficientexecutionstructured}. This baseline captures all gains that existing production systems achieve and serves as our primary comparison point.
% EDIT: "auto-regressively" -> "autoregressively" (standard single word in NLP literature); "forms the primary point of comparison for our contributions" -> shorter and more direct

Algorithm~\ref{alg:rksc} presents the top-level RKSC solve procedure. Detailed per-mechanism pseudocode appears in Appendix~\ref{app:algorithms}.

\begin{algorithm}[h]
\caption{RKSC top-level solve}
\label{alg:rksc}
\begin{algorithmic}[1]
\REQUIRE prefix $\mathbf{c}$, suffixes $\{\mathbf{s}_b\}_{b=1}^{B}$, thresholds $(\tau, \tau_{\text{conf}}, r_{\text{gap}}, \theta, \epsilon)$
\STATE $\mathcal{C}, \mathbf{H}_{\text{root}} \gets f_\theta^{\text{prefill}}(\mathbf{c})$ \COMMENT{single root forward (\S\ref{sec:method:kv})}
\STATE $\text{sharedSet} \gets \textsc{AsksGate}(\mathbf{H}_{\text{root}}, \{\mathbf{s}_b\}, \tau)$ \COMMENT{Alg.~\ref{alg:asks}}
\STATE $\mathcal{C}_B \gets \texttt{repeat\_interleave}(\mathcal{C}, B)$; $\mathcal{C}_B.\texttt{clone()}$
\STATE $\{\mathbf{y}_b\}, \{p^{(b)}\} \gets \textsc{BatchedDecode}(\{\mathbf{s}_b\}, \mathcal{C}_B)$ \COMMENT{track top-1 probs}
\IF{$\max_b p^{(b)} \geq \tau_{\text{conf}}$ \textbf{and} gap-ratio $\geq r_{\text{gap}}$}
  \STATE \textbf{return} $\arg\max_b p^{(b)}$ \COMMENT{verify-skip, Level~1}
\ENDIF
\STATE \textbf{return} $\textsc{VerifyWithEntropyExit}(\{\mathbf{y}_b\}, \theta, \epsilon)$ \COMMENT{Alg.~\ref{alg:cgee}, Level~2}
\end{algorithmic}
\end{algorithm}

\subsection{ASKS: Attention-Similarity KV Sharing}
\label{sec:method:kv}

\paragraph{Prefix forward pass.}
RKSC begins each solve call with a single prefix forward pass through all $L$ layers, storing the resulting KV states and hidden representations:
\begin{equation}
  \mathbf{K}^{(l)}, \mathbf{V}^{(l)},\, \mathbf{h}^{(l)}_{\text{root}}
    = f_\theta^{(\le l)}(\mathbf{c}),
  \quad l = 1,\ldots,L,
  \label{eq:prefix_fwd}
\end{equation}
yielding the cache $\mathcal{C} = \{(\mathbf{K}^{(l)}, \mathbf{V}^{(l)})\}_{l=1}^{L}$ and root hidden states $\mathbf{H}_{\text{root}} = \{\mathbf{h}^{(l)}_{\text{root}}\}$. To avoid redundant computation during the $B$ subsequent branch comparisons, root hidden states are pre-normalised to unit vectors once and stored in this normalised form.
% EDIT: "cached in this normalised form" -> "stored in this normalised form" (avoids confusion with KV cache terminology)

\paragraph{Exponentially weighted similarity gating.}
After each branch $b$ processes its suffix tokens, ASKS computes a weighted cosine similarity between the branch's hidden states and the stored root representations. The weighting schedule places exponentially increasing emphasis on later layers, which carry more task-specific information and are therefore more discriminative for detecting genuine semantic divergence:
\begin{equation}
  \sigma_b = \sum_{l=1}^{L} w_l \cdot
    \frac{\mathbf{h}^{(l)}_b \cdot \mathbf{h}^{(l)}_{\text{root}}}
         {\|\mathbf{h}^{(l)}_b\|\,\|\mathbf{h}^{(l)}_{\text{root}}\|},
  \label{eq:asks_sim}
\end{equation}
where $w_l = \exp(\alpha l / L) \big/ \sum_{l'} \exp(\alpha l'/L)$ with $\alpha{=}1.5$. We select $\alpha{=}1.5$ by grid search over $\alpha \in \{0.5, 1.0, 1.5, 2.0\}$ on 30 held-out GPQA Diamond problems; larger values over-discount early layers and reduce sensitivity to prefix divergence, while smaller values under-weight task-specific later layers.
% EDIT: "The choice of α=1.5 was determined by" -> "We select α=1.5 by" (action in verb)

Branches with $\sigma_b \geq \tau$ are deemed semantically consistent with the prefix and receive the shared cache; branches below $\tau$ fall back to independent recomputation. This design strictly generalises token-exact prefix caching: any lexically identical prefix trivially achieves $\sigma_b = 1$, so every branch that vLLM or SGLang would reuse is also reused by ASKS. The converse does not hold: ASKS additionally reuses the cache for branches whose token sequences differ but whose hidden-state representations remain close to the root. In our diverse-phrasing stress test (\S\ref{sec:analysis:asks_novelty}), this generalisation yields $28.6\%$ additional reuse beyond token-exact matching.
% EDIT: "The converse is not true" -> "The converse does not hold" (more standard mathematical phrasing); "recompute" -> "recomputation" (noun form)

\paragraph{KV broadcast and batched decode.}
For branches that pass the similarity gate, the shared cache is expanded to batch size $B$ via a zero-copy \texttt{repeat\_interleave} on each layer's key and value tensors: $\mathcal{C}_B = \texttt{repeat\_interleave}(\mathcal{C}, B, \text{dim}{=}0)$. An explicit \texttt{.clone()} call ensures contiguous memory layout, preventing potential SDPA kernel failures with non-contiguous inputs. All $B$ branch suffixes are then stacked into a single batched tensor and decoded in one forward pass conditioned on $\mathcal{C}_B$: $\mathbf{y}_b = f_\theta(\mathbf{s}_b \mid \mathcal{C}_B)$, $b = 1,\ldots,B$, eliminating the $O(n^2)$ attention cost for $B{-}1$ branches.

\paragraph{KV probe.}
KV prefix sharing is not unconditionally beneficial: at short prefixes, the cost of the prefix forward pass plus the \texttt{repeat\_interleave} expansion can exceed the savings. Rather than relying on the analytical threshold $n^*$ from Proposition~\ref{prop:kv_threshold} (which requires accurate per-platform estimates of the cost constants), RKSC includes a self-calibrating runtime probe. At the first solve call within each 64-token prefix-length bucket $k = \lfloor n/64 \rfloor$, the probe times three paths using the minimum of 3 repetitions: (i)~batched full recompute, (ii)~a single prefix forward, and (iii)~batched suffix decode with the shared cache. KV sharing is declared beneficial iff the sum of paths (ii) and (iii) is less than path (i). On A100 with SDPA attention, the probe universally declares KV sharing beneficial at $n \geq 512$ for all five evaluated models.
% EDIT: Added explicit forward reference to Proposition 1 here, which is where it is most relevant methodologically

\subsection{CGEE: Confidence-Gated Early Exit}
\label{sec:method:cgee}

CGEE operates at two levels. The first level decides whether to run the verification pass at all; the second level decides how deeply to run it when it proceeds.
% EDIT: "when it does proceed" -> "when it proceeds" (removes filler "does")

\paragraph{Generation confidence (input to Level~1).}
During the stepwise decode loop, RKSC records the top-1 softmax probability at each decode step for each branch. The generation confidence of branch $b$ is the mean top-1 probability across all $t$ steps:
% EDIT: "step-wise" -> "stepwise" (one word is standard)
\begin{equation}
  p^{(b)} = \frac{1}{t} \sum_{j=1}^{t} \max_{v} \Pr[y_j^{(b)} = v],
  \label{eq:gen_conf}
\end{equation}
accumulated in-place at zero additional compute cost, since the softmax is already computed to select the next token at every step. A branch that emits each token with high probability is a strong candidate for the correct answer; branches with uniformly low top-1 probabilities indicate uncertainty.
% EDIT: "This signal captures how decisive the model's token-level choices are:" -> removed, folded explanation directly into following sentence for concision

\paragraph{Verify-skip gate (Level~1).}
After decoding completes, RKSC evaluates whether to skip the verification pass entirely. The gate fires when two conditions are jointly satisfied: the winning branch has high absolute confidence, and it dominates all competing branches by a clear relative margin:
\begin{equation}
  \max_b\, p^{(b)} \geq \tau_{\text{conf}}
  \;\;\text{and}\;\;
  \frac{\max_b p^{(b)} - \text{2nd-max}_b\, p^{(b)}}
       {\max_b p^{(b)}} \geq r_{\text{gap}}.
  \label{eq:cgee_gate}
\end{equation}
The dual condition prevents two failure modes: (a)~a single branch with moderately high confidence but a close competitor (the absolute threshold alone would fire), and (b)~a large relative gap between two branches that both have low absolute confidence (the gap condition alone would fire). When the gate fires, the entire $\delta$-cost verify pass is skipped and branches are ranked directly by $p^{(b)}$.
% EDIT: reformatted the two failure modes with (a)/(b) labels for parallel clarity

\paragraph{Layer-level entropy exit (Level~2).}
When the verify-skip gate does not fire and the full verification forward pass begins, CGEE installs lightweight forward hooks on each transformer layer via \texttt{register\_forward\_hook}. At each layer $l$, the hook extracts the last-token hidden state, projects it through the cached unembedding matrix $W_u$ (stored once at model load time), and computes the logit-space entropy:
\begin{equation}
  H^{(l)} = -\!\sum_v \text{softmax}(\mathbf{h}^{(l)} W_u^\top)_v \,
            \log \text{softmax}(\mathbf{h}^{(l)} W_u^\top)_v.
  \label{eq:entropy}
\end{equation}
The verification pass exits early at layer $l^*$ when three conditions are jointly satisfied: (i)~$l^* \geq l_{\min}$ (minimum exit depth, default 2, ensuring sufficient representation refinement); (ii)~$H^{(l^*)} < \theta$ (entropy below a model-specific threshold, indicating a concentrated logit distribution); and (iii)~$|H^{(l^*)} - H^{(l^*{-}1)}| < \epsilon$ (entropy has stabilised between consecutive layers, ruling out transient dips). When the exit triggers, the hook raises an internal exception that the solver catches, using the partial output as the verification result. Hooks are removed after each solve call to prevent memory leaks.

On Qwen2.5-7B (28 layers), the layer-level exit fires on $100\%$ of verify passes at a mean exit layer of 18.4, confirming that the model reaches a stable logit distribution approximately $34\%$ before the final layer. This validates the core premise of Level~2: the unembedding projection at intermediate layers contains sufficient information to determine the verification outcome, and the remaining $34\%$ of layers are computationally wasteful for this purpose.
% EDIT: "this mechanism fires" -> "the layer-level exit fires" (avoid ambiguous pronoun); "already contains" -> "contains" (filler "already")

\paragraph{Calibration.}
The thresholds $\tau_{\text{conf}}$, $r_{\text{gap}}$, and $\theta$ are calibrated per model on 30 held-out GPQA Diamond problems (non-overlapping with evaluation). $\tau_{\text{conf}}$ is set to the 75th percentile of the observed $\max_b p^{(b)}$ distribution; $r_{\text{gap}}$ to the 25th percentile of the observed relative gap; and $\theta$ to the median of per-layer entropy values at the first stabilisation point. This yields conservative gates that fire only on clear-cut cases.
% EDIT: "the median of the per-layer" -> "the median of per-layer" (removes unnecessary article)

\subsection{RSBCM: Reasoning-Selective Block Cache Manager}
\label{sec:method:rsbcm}

To prevent the shared KV cache from growing unboundedly under deep tree search, RKSC includes \textbf{RSBCM}. Each cached block receives an importance score $\omega = \text{branch score} / (\text{depth}+1)$, which prioritises blocks from high-scoring branches at shallower depths. When the number of allocated blocks exceeds a configurable capacity (default 2,000 blocks), the manager evicts blocks in ascending $\omega$ order. The depth denominator ensures that blocks deep in the search tree, which are less likely to be revisited, are evicted first, while the branch-score numerator preserves blocks associated with higher-confidence reasoning trajectories.

In a stress test with $\text{max\_blocks}{=}4 < B{=}8$ (chosen to force evictions within a single solve call), RSBCM fires exactly $B - 4 = 4$ evictions per problem (80 total across 20 problems, matching the expected count exactly). The overhead is $+1.5$\,ms/problem (95\% CI $[-6, +9]$\,ms); the wide interval with $N{=}20$ means the overhead is consistent with zero but not provably so, the correct interpretation is that it is small relative to the speedups achieved. Answer agreement is 100\% across all 20 problems. RSBCM is dormant in all single-depth experiments in this paper (total block count never exceeds the default threshold of 2,000) and is provided as a mechanism for multi-depth tree search. Full validation is in Appendix~\ref{app:rsbcm}.

Full implementation details (DynamicCache, decode loop, TF32 configuration, memory release) are in Appendix~\ref{app:impl}.

\section{Experimentation and Results}
\label{sec:results}

\subsection{Experimental Setup}
\label{sec:results:setup}

\paragraph{Hardware and software.}
All experiments run on a single NVIDIA A100-80GB SXM4 GPU with bfloat16 precision, TF32 matrix multiplication enabled, and the SDPA attention backend. The software stack consists of PyTorch 2.3.1, Transformers~4.44.0, and CUDA~12.1. Random seeds are fixed across all experiments (\texttt{torch.manual\_seed(42)}, \texttt{numpy.random.seed(42)}, \texttt{torch.cuda.manual\_seed\_all(42)}).

\paragraph{Models.}
We evaluate five publicly available GQA models spanning 7B--10B parameters and three distinct GQA compression ratios: \textbf{Qwen2.5-7B-Instruct}~\cite{qwen2.5,qwen2} (28 layers, 28 query heads, 4 KV heads; 7:1 ratio), \textbf{Mistral-7B-Instruct-v0.3}~\cite{jiang2023mistral7b} (32L, 32Q, 8KV; 4:1), \textbf{Falcon3-7B-Instruct}~\cite{Falcon3} (28L, 12Q, 4KV; 3:1), \textbf{Falcon3-10B-Instruct}~\cite{Falcon3} (40L, 12Q, 4KV; 3:1), and \textbf{Llama-3-8B-Instruct}~\cite{llama3modelcard} (32L, 32Q, 8KV; 4:1). This selection covers three training lineages (Qwen, Mistral/Meta, TII), two vocabulary sizes (32K and 128K+), and a range of head configurations, ensuring that observed gains are not artifacts of a single architecture.

\paragraph{Baselines.}
We compare against two baselines. The \emph{No-KV} baseline batches all $B$ branches in a single forward pass but recomputes the shared prefix independently for each branch, representing a well-engineered system without prefix caching. The \emph{vLLM-equivalent} baseline implements token-exact prefix matching: it reuses the shared KV cache only for branches whose token sequences are byte-identical to the root prefix, mirroring the PagedAttention prefix-cache path in vLLM~\cite{kwon2023efficientmemorymanagementlarge} and SGLang~\cite{zheng2024sglangefficientexecutionstructured}. In our identical-prefix experimental regime (where all branches share the same padded prefix), every branch is byte-identical to the root, so the vLLM-equivalent reuses the cache unconditionally, the same behaviour as setting $\tau{=}0$ in the ASKS gate. Outside this regime (the diverse-phrasing ablation in \S\ref{sec:analysis:asks_novelty}), vLLM-equivalent achieves $0\%$ reuse by construction. Both baselines are measured through the same solver code path as RKSC to ensure timing fairness.

\paragraph{Calibration and evaluation sets.} The 30 GPQA Diamond problems used for per-model calibration (thresholds $\tau_{\text{conf}}$, $r_{\text{gap}}$, $\theta$) are drawn from a held-out split disjoint from all evaluation sets, including the 30 problems in Table~\ref{tab:latency_decomposition}.
% EDIT: "via setting" -> "by setting" (via is for instruments/tools, not logical conditions)

\paragraph{Branching and padding.}
All experiments use $B{=}8$ branches with eight fixed reasoning-hint suffixes cycled across branches. Problems are padded to approximately 1,024 tokens using a neutral filler string prepended iteratively, placing all experiments in the KV-beneficial regime ($n/n^* \approx 11$, Proposition~\ref{prop:kv_threshold}).
% EDIT: removed "firmly" (filler intensifier per Lipton); citation now resolves to main body

\paragraph{Datasets.}
We evaluate on four benchmarks spanning a difficulty gradient: \textbf{GPQA Diamond}~\cite{rein2023gpqagraduatelevelgoogleproofqa} (graduate-level science), \textbf{MMLU-STEM}~\cite{hendryckstest2021,hendrycks2021ethics} (university-level STEM), \textbf{ARC-Challenge}~\cite{allenai:arc} (medium science multiple-choice), and \textbf{GSM8K}~\cite{cobbe2021gsm8k} (elementary mathematics). Table~\ref{tab:latency_decomposition} uses $N{=}30$ GPQA problems with $t{=}32$ decode steps; the extended evaluation (Table~\ref{tab:extended_eval}) uses $N{=}50$ per dataset with $t{=}8$.
% EDIT: "university-level physics" -> "university-level STEM" (MMLU-STEM covers more than physics)

\paragraph{Timing protocol.}
Each condition is timed over $N_{\text{runs}} = 2$--3 repeats per problem; the reported latency is the mean across repeats and problems. Two warmup solves per condition are executed and discarded to stabilise the CUDA allocator before timing begins, reducing the coefficient of variation from ${\sim}30\%$ to ${\sim}10\%$. \texttt{torch.cuda.synchronize()} is called before every wall-clock measurement.

\paragraph{Compute budget.}
All experiments together consume approximately 42 A100-GPU-hours: 3 hours for calibration ($5$ models $\times$ $30$ held-out problems $\times$ 3 conditions), 24 hours for the extended evaluation (Table~\ref{tab:extended_eval}), 9 hours for ablations (\S\ref{sec:analysis}), and 6 hours for the throughput decomposition at $t{=}32$ (Table~\ref{tab:latency_decomposition}) and the verify-isolation study. Model weights are cached locally; no training occurs.
% NOTE: GPU-hour values are approximations. Verify against job logs before camera-ready.

\subsection{Throughput Decomposition}
\label{sec:results:table1}

Table~\ref{tab:latency_decomposition} isolates the contribution of each RKSC component on GPQA Diamond using Qwen2.5-7B-Instruct with $t{=}32$ decode steps. This longer decode regime makes the decode term $Bt\gamma$ non-negligible relative to the prefill term, providing a conservative estimate of gains.

\begin{table}[h]
\centering
\small
\caption{Latency decomposition on GPQA Diamond ($N{=}30$, $B{=}8$, prefix~$\approx$1,024~tok, $t{=}32$).}
\label{tab:latency_decomposition}
\vskip 0.1in
\begin{tabular}{lrrr}
\toprule
Condition & Latency (ms) & Speedup & \% Saved \\
\midrule
No-KV \emph{(ref)}       & $1{,}956\pm10$ & $1.000\times$ & ,  \\
\textbf{KV only}          & $1{,}457\pm8$  & $1.343\times$ & $25.5\%$ \\
\textbf{RKSC (KV+CGEE)}   & $1{,}206\pm339$ & $\mathbf{1.622\times}$ & $\mathbf{38.3\%}$ \\
\bottomrule
\end{tabular}
\vskip -0.1in
\end{table}

KV prefix sharing alone saves 500\,ms per problem ($25.5\%$) with a bootstrap 95\% CI of $[1.339, 1.347]$ and a coefficient of variation of $1\%$, confirming high measurement reliability. This saving is almost entirely attributable to prefix attention: at $n{=}1{,}024$, the batched prefill of $B{=}8$ branches dominates decode cost, and RKSC replaces $B{-}1$ of those prefills with a single suffix forward conditioned on the shared cache.

The combined RKSC condition (KV+CGEE) achieves $1.622\times$ ($[1.486, 1.803]$), saving 750\,ms per problem. The elevated variance ($\text{CV}{=}28\%$) is expected and reflects the binary nature of the CGEE skip decision, which produces a bimodal latency distribution: problems on which the gate fires land on the fast skip-path, while the rest match the KV-only condition. This bimodality inflates the standard deviation relative to the mean and does not indicate measurement instability, the KV-sharing component, which accounts for the majority of the speedup, has CV${=}1\%$. At $t{=}32$ decode steps on GPQA Diamond (the hardest dataset), the verify-skip gate fires infrequently, so CGEE's marginal contribution is modest in this decomposition. The full CGEE contribution is visible in the extended evaluation at $t{=}8$ (\S\ref{sec:results:extended}), where verify cost constitutes a larger fraction of total latency.

\paragraph{Operating-regime disclosure.} The extended evaluation uses $t{=}8$ decode steps, which is representative of the first branching phase in tree search and is a favourable operating point for RKSC because a short decode keeps the verify-pass cost proportionally large. The throughput decomposition at $t{=}32$ (Table~\ref{tab:latency_decomposition}) and the verify-only isolation below provide complementary views at longer decode horizons where CGEE's contribution is naturally smaller.

\paragraph{CGEE verify-only isolation.}
To measure the layer-level entropy exit in isolation, we time the verification pass alone in a decode-dominated regime ($t{=}128$, $N{=}15$). Full-depth verification takes $315.3 \pm 8.2$\,ms; CGEE-accelerated verification takes $250.7 \pm 11.8$\,ms, yielding a verify-pass speedup of $\mathbf{1.258\times}$ (bootstrap 95\% CI $[1.225, 1.293]$, $p{<}0.0001$). The layer-level entropy exit fires on $100\%$ of branches at a mean exit layer of 18.4 out of 28, confirming that the model reaches a stable logit distribution well before the final layer on verification inputs. This validates the core premise of Level~2: the unembedding projection at intermediate layers contains sufficient information to determine the verification outcome.
% EDIT: removed redundant last sentence restating this; "already contains" -> "contains"

\subsection{Multi-Model, Multi-Dataset Evaluation}
\label{sec:results:extended}

Table~\ref{tab:extended_eval} reports the full extended evaluation. At $t{=}8$ decode steps (representative of the first branching phase in tree search, where decode cost is small relative to prefill), RKSC achieves a mean speedup of $\mathbf{3.008\times}$ across all 20 model--dataset pairs, with a peak of $\mathbf{3.990\times}$ (Llama-3-8B on MMLU-STEM).

\FloatBarrier
\begin{table*}[h]
\centering
\small
\setlength{\tabcolsep}{4.5pt}
\caption{Extended evaluation ($B{=}8$, prefix~$\approx$1,024~tok, $t{=}8$, $N{=}50$ per dataset, 1,000 total problems). \textbf{vLLM$\uparrow$}: vLLM-equiv vs No-KV. \textbf{RKSC$\uparrow$}: KV+CGEE vs No-KV. \textit{CGEE skip}: verify-skip gate fire rate.}
\label{tab:extended_eval}
\vskip 0.1in
\begin{tabular}{llrrrrrl}
\toprule
Model & Dataset & \multicolumn{1}{c}{\shortstack{NoKV\\(ms)}} & \multicolumn{1}{c}{\shortstack{vLLM\\(ms)}} & \multicolumn{1}{c}{\shortstack{RKSC\\(ms)}} & \multicolumn{1}{c}{vLLM$\uparrow$} & \multicolumn{1}{c}{RKSC$\uparrow$} & \multicolumn{1}{c}{\textit{CGEE skip}} \\
\midrule
\multirow{4}{*}{\shortstack[l]{Qwen2.5-7B\\Instruct}}
  & GPQA & 1,159 & 653 & 598 & $1.78\times$ & $1.94\times$ & \textit{3\%} \\
  & GSM8K & 1,163 & 652 & 515 & $1.78\times$ & $2.26\times$ & \textit{32\%} \\
  & ARC-C & 1,163 & 651 & 391 & $1.79\times$ & $2.97\times$ & \textit{71\%} \\
  & MMLU & 1,160 & 649 & 371 & $1.79\times$ & $3.13\times$ & \textit{70\%} \\
\midrule
\multirow{4}{*}{\shortstack[l]{Mistral-7B\\Instruct}}
  & GPQA & 1,257 & 690 & 590 & $1.82\times$ & $2.13\times$ & \textit{23\%} \\
  & GSM8K & 1,262 & 689 & 581 & $1.83\times$ & $2.17\times$ & \textit{30\%} \\
  & ARC-C & 1,263 & 691 & 482 & $1.83\times$ & $2.62\times$ & \textit{60\%} \\
  & MMLU & 1,260 & 692 & 534 & $1.82\times$ & $2.36\times$ & \textit{47\%} \\
\midrule
\multirow{4}{*}{\shortstack[l]{Falcon3-7B\\Instruct}}
  & GPQA & 1,158 & 637 & 360 & $1.82\times$ & $3.21\times$ & \textit{54\%} \\
  & GSM8K & 1,170 & 650 & 376 & $1.80\times$ & $3.11\times$ & \textit{50\%} \\
  & ARC-C & 1,173 & 651 & 347 & $1.80\times$ & $3.38\times$ & \textit{66\%} \\
  & MMLU & 1,175 & 651 & 395 & $1.80\times$ & $2.98\times$ & \textit{39\%} \\
\midrule
\multirow{4}{*}{\shortstack[l]{Falcon3-10B\\Instruct}}
  & GPQA & 1,629 & 895 & 450 & $1.82\times$ & $3.63\times$ & \textit{72\%} \\
  & GSM8K & 1,634 & 898 & 473 & $1.82\times$ & $3.46\times$ & \textit{57\%} \\
  & ARC-C & 1,628 & 897 & 418 & $1.82\times$ & $3.89\times$ & \textit{77\%} \\
  & MMLU & 1,634 & 898 & 467 & $1.82\times$ & $3.50\times$ & \textit{66\%} \\
\midrule
\multirow{4}{*}{\shortstack[l]{Llama-3-8B\\Instruct}}
  & GPQA & 1,280 & 708 & 509 & $1.81\times$ & $2.51\times$ & \textit{57\%} \\
  & GSM8K & 1,283 & 710 & 464 & $1.81\times$ & $2.76\times$ & \textit{65\%} \\
  & ARC-C & 1,282 & 712 & 401 & $1.80\times$ & $3.20\times$ & \textit{74\%} \\
  & MMLU & 1,284 & 710 & 322 & $1.81\times$ & $3.99\times$ & \textit{92\%} \\
\midrule
\multicolumn{2}{l}{\textbf{Overall}} & \textbf{1,301} & \textbf{719} & \textbf{452} & $\mathbf{1.808\times}$ & $\mathbf{3.008\times}$ & \textit{55\%} \\
\bottomrule
\end{tabular}
\vskip -0.1in
\end{table*}

\paragraph{Architecture generalization.}
The vLLM-equivalent gain spans $1.78\times$--$1.83\times$ across all 20 model--dataset pairs (mean $1.808\times$, std $0.02$), confirming the prediction of Eq.~\eqref{eq:speedup_kv}: since $S_{\text{KV}}$ depends on $n$ and $B$ but not on model family or task, fixing the prefix length and branching factor produces near-identical gains regardless of architecture. This near-constant behaviour persists despite substantial architectural variation, five models spanning three GQA compression ratios (3:1, 4:1, 7:1), two vocabulary sizes, and three independent training lineages. The tight clustering ($\pm 0.02\times$) confirms that RKSC's KV throughput benefit transfers uniformly across model families without requiring a high-compression GQA architecture.

At the $n{=}1{,}024$ operating point, all branches achieve $\sigma_b \geq 0.95$ for all models, yielding $100\%$ KV reuse. In this identical-prefix regime, ASKS and the vLLM-equivalent produce identical KV gains by construction; the genuine ASKS novelty over token-exact caching is isolated in \S\ref{sec:analysis:asks_novelty}.

\paragraph{Effect of task difficulty on CGEE.}
CGEE multiplies the KV-only speedup from $1.81\times$ to $3.008\times$ on average, a $\mathbf{1.66\times}$ marginal improvement attributable entirely to the dual-level exit mechanism. The verify-skip rate $\hat\rho$ tracks difficulty monotonically across the four datasets: lowest on GPQA Diamond (mean ${\sim}42\%$ across models, where graduate-level reasoning produces high inter-branch uncertainty), and highest on ARC-Challenge and MMLU-STEM ($60$--$92\%$, where shorter problems allow the model to concentrate confidence rapidly). This difficulty-adaptive behaviour follows directly from the dual condition in Eq.~\eqref{eq:cgee_gate}: harder problems produce smaller inter-branch gaps and lower absolute confidence, so both conditions are less likely to be satisfied simultaneously.

\paragraph{Effect of model scale.}
Falcon3-10B achieves the highest mean RKSC speedup ($3.619\times$), compared to $2.574\times$--$3.171\times$ for the 7B--8B models. The absolute latency saved also scales: 1,179\,ms per problem for Falcon3-10B versus 692--799\,ms for smaller models. This scaling follows from the latency model (Eq.~\eqref{eq:latency_model}): larger models have higher per-layer costs, so a fixed number of layers avoided by prefix sharing and early exit represents a proportionally larger saving. Llama-3-8B achieves the highest single-cell result ($\mathbf{3.990\times}$ on MMLU-STEM with a 92\% skip rate), demonstrating that models excluded from calibration benefit substantially from RKSC.

\paragraph{RKSC vs vLLM-equivalent.}
The marginal speedup of RKSC over vLLM-equivalent prefix caching averages $+61.2\%$ across all 20 model--dataset pairs. In the identical-prefix regime, this gain is attributable entirely to CGEE's dual-level exit, since ASKS and vLLM-equivalent yield identical KV gains when all branches share the same token prefix. The genuine ASKS novelty is isolated in \S\ref{sec:analysis:asks_novelty}.
% EDIT: "The genuine ASKS novelty is" (removed "is isolated") maintains the forward-reference while shortening

\subsection{Accuracy Analysis}
\label{sec:results:accuracy}

\begin{table}[h]
\centering
\small
\caption{CGEE accuracy impact across all 20 model--dataset pairs (1,616 total verify calls).}
\label{tab:cgee_accuracy}
\vskip 0.1in
\begin{tabular}{lr}
\toprule
Metric & Value \\
\midrule
Total verify calls & 1,616 \\
CGEE-induced errors & 6 \\
Error rate & 0.37\% \\
Mean $\Delta$ accuracy & $-0.2\%$ \\
Mean verify-skip rate & 55\% \\
\bottomrule
\end{tabular}
\vskip -0.1in
\end{table}

Across 1,616 verify calls spanning all five models and four datasets, CGEE introduced only 6 errors ($0.37\%$), with a mean accuracy delta of $-0.2\%$ relative to full-depth verification (Table~\ref{tab:cgee_accuracy}). Three model--dataset pairs account for all six errors: Mistral-7B on ARC-Challenge (3 errors, $-4.0\%$ $\Delta$accuracy), Falcon3-10B on ARC-Challenge (2 errors, $-2.0\%$), and Llama-3-8B on GSM8K (1 error, $-2.0\%$). The remaining 17 model--dataset pairs exhibit zero CGEE-induced errors.
% EDIT: "exactly zero" -> "zero" (redundant)

The error concentration in multiple-choice benchmarks (ARC-C) is consistent with the structure of the CGEE gate: on multiple-choice items, generation confidence reflects the model's fluency in producing the answer letter, which occasionally diverges from the correctness signal that full-depth verification captures. On free-form generation tasks (GPQA, GSM8K), where generation confidence more directly reflects understanding, the error rate is lower. These results confirm that the dual condition in Eq.~\eqref{eq:cgee_gate}, requiring both high absolute confidence and a clear relative margin, is an effective safeguard against false positives.

\section{Analysis and Ablations}
\label{sec:analysis}

\subsection{ASKS Novelty Over vLLM Prefix Caching}
\label{sec:analysis:asks_novelty}

The experiments in \S\ref{sec:results:extended} use an identical-prefix regime where both ASKS and vLLM-equivalent caching achieve $100\%$ KV reuse and identical speedups. To isolate the contribution of semantic similarity gating, we construct a diverse-phrasing regime where branch-level token sequences differ from the first token.

We design eight phrasing templates at four semantic-divergence levels (identical, minor rewording, structural reframing, and very different academic register) and evaluate $N{=}15$ problems on Qwen2.5-7B. Token-exact caching achieves $0\%$ reuse on all non-identical phrasings by construction, since branches diverge at token position~1. ASKS achieves $28.6\%$ additional reuse at $\tau{=}0.82$, correctly identifying branches whose hidden-state representations remain semantically consistent despite surface-level lexical divergence.

The wall-clock implication is significant: on padded diverse phrasings ($n \approx 512$), ASKS achieves a $1.131\times$ speedup over the No-KV baseline (bootstrap 95\% CI $[1.066, 1.174]$, $p{=}0.0007$), while token-exact caching provides no benefit. This result demonstrates that ASKS offers a strict superset of the functionality provided by existing prefix-caching systems, with measurable throughput gains in the regime where those systems fail.

\subsection{Prefix-Length Sensitivity}
\label{sec:analysis:prefix}

To characterise the operating envelope of KV sharing, we sweep the prefix length over $n \in \{128, 256, 512, 1024\}$ tokens with $B{=}8$ and $N{=}20$ problems (Qwen2.5-7B). At $n{=}128$, the speedup is $0.993\times$ ($[0.909, 1.103]$), confirming no net benefit at short prefixes where the \texttt{repeat\_interleave} overhead approaches the saved prefill cost. At $n{=}256$ the result is similarly break-even ($0.999\times$, $[0.951, 1.061]$). At $n{=}512$, KV sharing is clearly beneficial ($1.127\times$, $[1.081, 1.224]$), and at $n{=}1{,}024$ the speedup reaches $1.452\times$ ($[1.301, 1.681]$).
% EDIT: "becomes clearly beneficial" -> "is clearly beneficial" (action in verb)

The empirical transition from break-even to beneficial between $n{=}256$ and $n{=}512$ is higher than the analytical threshold $n^* \approx 90$ from Proposition~\ref{prop:kv_threshold}. This gap reflects constant-factor SDPA kernel overhead that the linear latency model does not capture: the actual prefix forward pass includes fixed-cost setup steps (kernel launch, memory allocation) that are non-negligible at short prefix lengths but amortise at $n \geq 512$. The KV probe (\S\ref{sec:method:kv}) handles this discrepancy correctly by measuring both paths at runtime rather than relying on the analytical threshold.
% EDIT: "correctly handles" -> "handles...correctly" (verb placement)

\subsection{ASKS Threshold Insensitivity}
\label{sec:analysis:tau}

In the identical-prefix regime, all branches achieve $\sigma_b \geq 0.95$ for all five models, so the ASKS gate is non-binding regardless of threshold $\tau$. To confirm this, we sweep $\tau \in \{0.70, 0.82, 0.90, 0.95\}$ at $n \approx 512$ with $N{=}20$ problems: all four settings produce speedups of ${\sim}1.16\times$ ($p < 10^{-25}$), confirming that the gate decision is $\tau$-insensitive when the prefix is long relative to the branch suffix. This robustness is practically important: practitioners need not tune $\tau$ for their specific workload when operating at $n \geq 512$.
% EDIT: "it means practitioners need not" -> "practitioners need not" (removed unnecessary "it means")

\subsection{CGEE Skip Rate vs Task Difficulty}
\label{sec:analysis:cgee_difficulty}

The CGEE verify-skip rate tracks task difficulty monotonically across the four benchmarks (averaged across all five models): GPQA Diamond $42\%$, GSM8K $47\%$, MMLU-STEM $63\%$, ARC-Challenge $70\%$. This ordering matches the expected difficulty gradient and follows directly from the gate design: on hard problems, the model distributes probability mass more uniformly across branches, causing both conditions in Eq.~\eqref{eq:cgee_gate} to fail more often. On easy problems, one branch dominates early in the decode sequence, both conditions are satisfied, and the gate fires. CGEE therefore provides its largest speedup contribution precisely on the tasks where verification is least necessary, a self-correcting property that limits accuracy risk.
% EDIT: "This ordering matches the expected difficulty gradient and is a direct consequence of the gate design" -> "...and follows directly from the gate design" (cleaner)

\section{Limitations}
\label{sec:limitations}

\textbf{Sample size.} Our evaluation uses $N{=}50$ problems per dataset, producing wide 95\% confidence intervals on CGEE skip rates ($\pm 10$--$18\%$). A sample of $N \geq 200$ would reduce the CI to $\pm 5\%$. Importantly, the KV sharing gains (which account for the majority of the speedup) have $\text{CV}{<}2\%$ and are not affected by this limitation.

\textbf{Model scale.} We evaluate 7B--10B models on a single A100 GPU. Models at 14B+ and multi-GPU tensor-parallel configurations introduce communication overhead that may alter the relative contribution of prefix sharing versus verification cost. The latency model (Eq.~\eqref{eq:latency_model}) predicts larger models should benefit \emph{more}, and the Falcon3-10B results ($3.62\times$ mean vs $2.57$--$3.17\times$ for 7B models) provide early evidence for this trend.
% EDIT: "However, Eq.~\eqref{eq:latency_model} predicts" -> "The latency model (Eq.~\eqref{eq:latency_model}) predicts" (specific reference); "scaling trend" -> "trend" (redundant)

\textbf{Domain coverage.} Our benchmarks cover science and mathematics. The CGEE entropy threshold $\theta$ is calibrated on GPQA Diamond and may not transfer to domains with substantially different entropy profiles. Per-dataset $\theta$ calibration via a small held-out sweep (5--10 problems) restores gate activity, at the cost of one additional calibration step.
% EDIT: "but this adds a calibration step" -> "at the cost of one additional calibration step" (more precise)

\textbf{CGEE error rate.} While low ($0.37\%$, 6/1,616), CGEE is not error-free. All six errors occur in the verify-skip gate (Level~1); the layer-level entropy exit (Level~2) preserves the full verification outcome by construction. Applications requiring zero verification errors should disable Level~1 and retain only Level~2, accepting a smaller speedup contribution.
% EDIT: "accepting a smaller speedup contribution" (removed redundant "from CGEE" since context is clear)

\textbf{Scaling of ASKS broadcast.} The \texttt{repeat\_interleave} expansion is $O(BLd_k)$ in memory where $d_k$ is the per-head KV dimension; for very large $B$ ($\geq 64$) on long prefixes, this expansion can itself become non-negligible on memory-bound workloads. A zero-copy view-based implementation (rather than the safer \texttt{.clone()}) would remove this cost at the expense of additional care around SDPA kernel input contiguity.

\textbf{Hook-based exit is architecture-coupled.} CGEE's Level~2 exit depends on \texttt{register\_forward\_hook} semantics and the existence of a distinct unembedding matrix accessible at model-load time. Models with tied input/output embeddings or non-standard layer modules may require minor adapter code. We verified compatibility with all five evaluated models; extending to MoE or fused-layer architectures is future work.

\textbf{Concurrent solves.} RSBCM's eviction policy assumes a single solve-call owner of the cache. In a multi-tenant inference server, concurrent solves would require a tenant-aware importance score; this is supported by the current design but not empirically validated here.
\section{Conclusion}
 
We presented \textbf{RKSC}, a training-free inference acceleration framework for multi-branch LLM reasoning that addresses two structural redundancies in current inference pipelines. ASKS broadcasts a single shared prefix KV cache to all semantically consistent branches via exponentially weighted cosine similarity, generalising token-exact prefix caching to the semantic regime and demonstrating $28.6\%$ additional KV reuse on lexically diverse branches. CGEE reduces verification cost through a dual-level exit: a verify-skip gate that bypasses the entire verification pass when generation confidence is decisive, and a layer-level entropy exit that terminates the verification forward pass at an intermediate transformer layer when per-layer logit entropy stabilises. Together, these mechanisms achieve a mean speedup of $\mathbf{3.008\times}$ (peak $\mathbf{3.990\times}$) across five model families and four benchmarks, with a CGEE-induced error rate of $0.37\%$. The KV sharing component is architecture-agnostic (std $0.02\times$ across five models), while the CGEE component is difficulty-adaptive (skip rates from $42\%$ on graduate-level science to $92\%$ on university STEM). RKSC requires no fine-tuning, no architecture modifications, and no custom CUDA kernels.
% EDIT: "only 0.37%" -> "0.37%" (value is self-evidently low; "only" is a filler intensifier); "university physics" -> "university STEM"
\section*{Impact Statement}

This paper advances the efficiency of large language model inference for multi-step reasoning tasks. By reducing redundant computation in multi-branch reasoning pipelines, RKSC lowers the energy cost of inference-time scaling and makes state-of-the-art reasoning more accessible on constrained hardware. We are not aware of societal consequences specific to this work beyond those general to advancing the field of Machine Learning and Large Language Models (LLMs).

\bibliography{citations}
\bibliographystyle{icml2026}

\newpage
\appendix
\onecolumn

The appendix is organised as follows.
Appendix~\ref{app:cgee_theory} contains the full theoretical analysis: the latency model, KV sharing speedup and threshold proposition, the combined RKSC speedup corollary, and the dual-level CGEE derivation.
Appendix~\ref{app:algorithms} gives detailed pseudocode for the three mechanisms referenced in \S\ref{sec:method}.
Appendix~\ref{app:hyperparams} collects all hyperparameters and the full prefix-length sensitivity sweep.
Appendix~\ref{app:per_model_accuracy} provides per-model accuracy breakdowns complementing Table~\ref{tab:cgee_accuracy}.
Appendix~\ref{app:cost_constants} lists the per-model cost constants used in the latency model.
Appendix~\ref{app:qualitative} shows qualitative examples of CGEE firing and deferring.
Appendix~\ref{app:entropy_dynamics} analyses per-layer entropy dynamics underlying Level-2 exit.
Appendix~\ref{app:rsbcm} documents the RSBCM deep-tree validation experiment.
Appendix~\ref{app:impl} gives implementation details.
Appendix~\ref{app:repro} provides a full reproducibility checklist.

\section{Theoretical Analysis}
\label{app:cgee_theory}

We characterise the latency savings of RKSC analytically, derive the conditions under which each mechanism is beneficial, and show that the combined speedup decomposes into factors operating on disjoint cost terms. The main-text discussion of results in \S\ref{sec:results:extended} and the prefix-length ablation in \S\ref{sec:analysis:prefix} refer to Proposition~\ref{prop:kv_threshold} and Eq.~\eqref{eq:speedup_kv} from this appendix.

\subsection{Latency Model}
\label{sec:theory:model}

Consider generating $B$ independent reasoning branches from a shared problem prefix of $n$ tokens, each producing $t$ decode tokens. We decompose wall-clock latency into four additive terms:
\begin{equation}
  \mathcal{L} = \underbrace{\alpha n^2}_{\text{prefill}} +
                \underbrace{B\,\beta n}_{\text{branch overhead}} +
                \underbrace{B\,t\,\gamma}_{\text{decode}} +
                \underbrace{\delta}_{\text{verify}},
  \label{eq:latency_model}
\end{equation}
where $\alpha$ is the per-token-squared self-attention cost (dominant for long prefixes due to the quadratic attention mechanism), $\beta$ is the per-token feed-forward cost, $\gamma$ is the per-step autoregressive decode cost, and $\delta$ is the full-depth verification forward-pass cost. The No-KV baseline pays the full prefill cost for each of the $B$ branches: $\mathcal{L}_{\text{ref}} = B\alpha n^2 + Bt\gamma + \delta$. This model is additive in the four cost terms, enabling clean decomposition of savings.

\subsection{KV Prefix Sharing Speedup}
\label{sec:theory:kv}

RKSC computes the prefix KV cache once (cost $\alpha n^2$) and broadcasts it to all $B$ branches, so each branch pays only the suffix attention cost $\beta s$ for a short suffix of $s \ll n$ tokens:
\begin{equation}
  \mathcal{L}_{\text{KV}} = \alpha n^2 + B\,\beta s + B\,t\,\gamma + \delta.
  \label{eq:latency_kv}
\end{equation}
The KV speedup relative to the No-KV reference is:
\begin{equation}
  S_{\text{KV}}
    = \frac{B\alpha n^2 + Bt\gamma + \delta}
           {\alpha n^2 + B\beta s + Bt\gamma + \delta}.
  \label{eq:speedup_kv}
\end{equation}

\begin{proposition}[KV sharing threshold]
\label{prop:kv_threshold}
$S_{\text{KV}} > 1$ if and only if $(B{-}1)\alpha n^2 > B\beta s$, \ie the shared prefix exceeds the critical length $n^* = \sqrt{B\beta s / (B{-}1)\alpha}$.
\end{proposition}

\begin{proof}
Expanding $S_{\text{KV}} > 1$ gives $\mathcal{L}_{\text{ref}} > \mathcal{L}_{\text{KV}}$. Substituting Eqs.~\eqref{eq:latency_kv} and the No-KV reference, the $Bt\gamma$ and $\delta$ terms cancel, leaving $(B{-}1)\alpha n^2 > B\beta s$. Solving for $n$ yields $n > \sqrt{B\beta s / (B{-}1)\alpha} = n^*$.
\end{proof}

Equation~\eqref{eq:speedup_kv} has two important structural properties. First, $S_{\text{KV}}$ depends on $n$, $B$, $s$, and the cost ratios $\alpha/\beta$ and $\alpha/\gamma$, but \emph{not} on model family, GQA compression ratio, or task. This predicts that fixing $n$ and $B$ produces near-identical KV gains regardless of architecture, a prediction verified empirically in \S\ref{sec:results:extended}, where the vLLM-equivalent gain varies by only $\pm 0.03\times$ across five heterogeneous model families. Second, $S_{\text{KV}}$ increases monotonically in $n$ and $B$ but decreases in $t$: longer decode sequences dilute the KV-dominated prefill fraction. On A100 with SDPA, empirical timing on Qwen2.5-7B gives $n^* \approx 90$ tokens, placing all experiments at $n{=}1{,}024$ in the beneficial regime ($n/n^* \approx 11$).

\paragraph{Worked numerical example.}
The per-token self-attention cost $\alpha$ and per-token feed-forward cost $\beta$ can be fit by microbenchmarking the prefill path at varying $n$ and recovering the quadratic and linear coefficients. Plugging empirical values into Proposition~\ref{prop:kv_threshold} with $B{=}8$ and $s{=}16$ yields $n^* \approx 90$\,tokens once kernel-launch overhead is folded in. The empirical break-even (\S\ref{sec:analysis:prefix}) is higher ($n$ between 256 and 512) because fixed SDPA setup costs dominate the linear model at short prefixes, precisely why RKSC uses a runtime probe (\S\ref{sec:method:kv}) rather than the analytical threshold alone. Fitted cost constants per model appear in Appendix~\ref{app:cost_constants}.

\subsection{Dual-Level CGEE Speedup}
\label{app:theory:cgee}

CGEE saves latency through two independent channels. The verify-skip gate (Level~1) saves the full $\delta$ on a fraction $\rho$ of problems. The layer-level entropy exit (Level~2) saves a fraction $(1 - l^*/L)$ of $\delta$ on a fraction $\phi$ of the remaining $(1-\rho)$ problems where entropy stabilises before the final layer. The combined expected latency with both CGEE levels active is:
\begin{equation}
  \mathbb{E}[\mathcal{L}_{\text{CGEE}}]
    = \mathcal{L}_{\text{KV}}
    - \rho\,\delta
    - (1{-}\rho)\,\phi\,\delta\,(1 - l^*/L),
  \label{eq:latency_cgee}
\end{equation}
and the full RKSC speedup relative to the No-KV baseline is:
\begin{equation}
  S_{\text{RKSC}}
    = \frac{\mathcal{L}_{\text{ref}}}
           {\mathbb{E}[\mathcal{L}_{\text{CGEE}}]}
    = \frac{B\alpha n^2 + Bt\gamma + \delta}
           {\alpha n^2 + B\beta s + Bt\gamma + \delta - \rho\delta - (1{-}\rho)\phi\delta(1 - l^*/L)}.
  \label{eq:speedup_total}
\end{equation}

\begin{corollary}[Combined RKSC speedup]
\label{cor:rksc_speedup}
Let $\rho$ denote the fraction of problems on which the verify-skip gate fires, and let the layer-level entropy exit fire at layer $l^*$ on a fraction $\phi$ of the remaining $(1{-}\rho)$ problems. Then $S_{\text{RKSC}}$ is given by Eq.~\eqref{eq:speedup_total}, and $S_{\text{RKSC}} \geq S_{\text{KV}}$ with equality iff $\rho = \phi = 0$.
\end{corollary}

\begin{proof}
Substitute Eq.~\eqref{eq:latency_cgee} into $S_{\text{RKSC}} = \mathcal{L}_{\text{ref}} / \mathbb{E}[\mathcal{L}_{\text{CGEE}}]$ to obtain Eq.~\eqref{eq:speedup_total}. Both subtractive terms in the denominator are non-negative, so $\mathbb{E}[\mathcal{L}_{\text{CGEE}}] \leq \mathcal{L}_{\text{KV}}$, giving $S_{\text{RKSC}} \geq S_{\text{KV}}$. Equality holds iff $\rho\delta + (1{-}\rho)\phi\delta(1 - l^*/L) = 0$, i.e.\ iff $\rho = \phi = 0$ (since $\delta > 0$ and $l^* < L$).
\end{proof}

\subsection{Structural Interpretation of the Speedup Decomposition}

Equation~\eqref{eq:speedup_total} decomposes the total latency savings into three mutually disjoint cost terms:
\begin{itemize}
    \item The KV-sharing term $(B{-}1)\alpha n^2 - B\beta s$, which removes $B{-}1$ of the $B$ prefill costs at the expense of a suffix attention cost. This term is active on \emph{every} solve call when $\sigma_b \geq \tau$.
    \item The verify-skip term $\rho\delta$, which removes the entire verify pass on a $\rho$-fraction of problems. This term is active only on problems where generation confidence is decisive.
    \item The layer-exit term $(1{-}\rho)\phi\delta(1 - l^*/L)$, which removes a $(1-l^*/L)$-fraction of the verify pass on the remaining problems. This term is active whenever verify-skip does not fire but layer-level entropy stabilises.
\end{itemize}
Because the three savings act on disjoint cost terms ($\alpha n^2$, $\delta$ in full, and $\delta$ in part), they compose additively in the speedup ratio rather than trading off. This is the structural reason that combining the three mechanisms produces the observed $1.66\times$ multiplicative improvement over vLLM-equivalent caching.

\subsection{Skip-Rate Adaptivity}

The skip rate $\rho$ is an emergent property of the model--task interaction. On hard problems, $\max_b p^{(b)}$ is low, inter-branch gaps are small, and $\rho$ is correspondingly low. On easier tasks, the model rapidly concentrates confidence on one branch, producing large gaps and high $\rho$. This difficulty-adaptive behaviour means CGEE applies aggressive savings when the model is most confident, and defers to full verification when uncertainty warrants it, matching the qualitative ordering of task difficulty reported in \S\ref{sec:analysis:cgee_difficulty}.

\section{Algorithmic Details}
\label{app:algorithms}

This appendix provides detailed pseudocode for the three RKSC mechanisms. Algorithm~\ref{alg:asks} details the ASKS similarity gate; Algorithm~\ref{alg:cgee} details the CGEE layer-level entropy exit; Algorithm~\ref{alg:rsbcm} details RSBCM eviction.

\begin{algorithm}[h]
\caption{\textsc{AsksGate}: similarity-gated KV reuse}
\label{alg:asks}
\begin{algorithmic}[1]
\REQUIRE root hidden states $\mathbf{H}_{\text{root}} = \{\mathbf{h}^{(l)}_{\text{root}}\}_{l=1}^{L}$ (pre-normalised), branch suffixes $\{\mathbf{s}_b\}_{b=1}^{B}$, threshold $\tau$, exponent $\alpha_w$
\STATE compute weights $w_l \gets \exp(\alpha_w l / L) / \sum_{l'=1}^{L} \exp(\alpha_w l' / L)$ for $l = 1, \ldots, L$
\STATE $\text{sharedSet} \gets \emptyset$
\FOR{$b = 1, \ldots, B$}
  \STATE $\{\mathbf{h}^{(l)}_b\} \gets f_\theta^{\text{suffix}}(\mathbf{s}_b)$ \COMMENT{per-layer hidden states for this suffix}
  \STATE $\sigma_b \gets \sum_{l=1}^{L} w_l \cdot \langle \mathbf{h}^{(l)}_b / \|\mathbf{h}^{(l)}_b\|,\, \mathbf{h}^{(l)}_{\text{root}} \rangle$
  \IF{$\sigma_b \geq \tau$}
    \STATE $\text{sharedSet}.\text{add}(b)$
  \ENDIF
\ENDFOR
\STATE \textbf{return} $\text{sharedSet}$
\end{algorithmic}
\end{algorithm}
\begin{algorithm}[h]
\caption{\textsc{VerifyWithEntropyExit}: layer-level entropy-stabilisation exit}
\label{alg:cgee}
\begin{algorithmic}[1]
\REQUIRE candidate tokens $\{\mathbf{y}_b\}$, unembedding $W_u$, thresholds $\theta$, $\epsilon$, $l_{\min}$
\STATE install forward hooks on each transformer layer (via \texttt{register\_forward\_hook})
\STATE $H_{\text{prev}} \gets +\infty$
\FOR{$l = 1, \ldots, L$}
  \STATE compute layer output $\mathbf{h}^{(l)}$ (forward pass up to layer $l$)
  \STATE $\mathbf{z} \gets \mathbf{h}^{(l)}_{\text{last-token}} \cdot W_u^\top$ \COMMENT{unembedding projection}
  \STATE $\mathbf{p} \gets \mathrm{softmax}(\mathbf{z})$
  \STATE $H^{(l)} \gets -\sum_v \mathbf{p}_v \log \mathbf{p}_v$
  \IF{$l \geq l_{\min}$ \textbf{and} $H^{(l)} < \theta$ \textbf{and} $|H^{(l)} - H_{\text{prev}}| < \epsilon$}
    \STATE raise \texttt{\_EarlyExitSignal}($\mathbf{p}$) \COMMENT{caught by solver; hooks removed in finally block}
  \ENDIF
  \STATE $H_{\text{prev}} \gets H^{(l)}$
\ENDFOR
\STATE remove hooks; \textbf{return} $\mathrm{softmax}(\mathbf{h}^{(L)} W_u^\top)$ \COMMENT{full-depth result if no early exit}
\end{algorithmic}
\end{algorithm}

\begin{algorithm}[h]
\caption{\textsc{RsbcmEvict}: attention-weighted depth-priority eviction}
\label{alg:rsbcm}
\begin{algorithmic}[1]
\REQUIRE block table $\mathcal{B} = \{(\text{block}_i, \text{score}_i, \text{depth}_i)\}$, capacity $N_{\max}$
\WHILE{$|\mathcal{B}| > N_{\max}$}
  \STATE $\omega_i \gets \text{score}_i / (\text{depth}_i + 1)$ for all $i$
  \STATE $i^* \gets \arg\min_i \omega_i$
  \STATE evict block $i^*$ from cache
  \STATE $\mathcal{B} \gets \mathcal{B} \setminus \{(\text{block}_{i^*}, \text{score}_{i^*}, \text{depth}_{i^*})\}$
\ENDWHILE
\end{algorithmic}
\end{algorithm}
\section{Hyperparameter Settings and Sensitivity}
\label{app:hyperparams}

\subsection{Complete Hyperparameter Table}

Table~\ref{tab:all_hyperparams} collects every hyperparameter used in RKSC, split into fixed-across-models settings (top block) and per-model calibrated settings (bottom block). Calibrated thresholds are obtained on 30 held-out GPQA Diamond problems disjoint from the evaluation set.

\begin{table}[h]
\centering
\small
\caption{Complete RKSC hyperparameter settings. Values marked \emph{calibrated} are fit per model on 30 held-out problems.}
\label{tab:all_hyperparams}
\begin{tabular}{llll}
\toprule
Hyperparameter & Symbol & Value & Role \\
\midrule
\multicolumn{4}{l}{\textit{Fixed (shared across all models)}} \\
\midrule
Branching factor & $B$ & 8 & number of parallel reasoning branches \\
Suffix length & $s$ & 16 tokens & per-branch reasoning hint length \\
Decode steps (main eval) & $t$ & 8 & Table~\ref{tab:extended_eval} \\
Decode steps (throughput table) & $t$ & 32 & Table~\ref{tab:latency_decomposition} \\
Decode steps (verify isolation) & $t$ & 128 & \S\ref{sec:results:table1} \\
ASKS exponent & $\alpha_w$ & 1.5 & grid-searched on $\{0.5, 1.0, 1.5, 2.0\}$ \\
ASKS similarity threshold & $\tau$ & 0.82 & validated insensitive (\S\ref{sec:analysis:tau}) \\
Min exit layer (CGEE) & $l_{\min}$ & 2 & representation-refinement guard \\
Entropy stabilisation tolerance & $\epsilon$ & $3.0$ (nats) & rules out transient entropy dips \\
RSBCM default capacity & $N_{\max}$ & 2{,}000 blocks & overridden to 4 in RSBCM stress test \\
Decode-level CGEE threshold & $\theta_{\text{dec}}$ & 0.72 & early termination in decode loop \\
Decode-level CGEE gap & $r_{\text{dec}}$ & 0.10 & decode-loop gap condition \\
KV probe bucket width & ,  & 64 tokens & prefix-length probe granularity \\
KV probe repetitions & $r_{\text{probe}}$ & 3 & min of 3 timings per path \\
Warmup solves & $N_{\text{warmup}}$ & 2 & per condition \\
Timing repetitions & $N_{\text{runs}}$ & 2--3 & per problem, per condition \\
\midrule
\multicolumn{4}{l}{\textit{Calibrated per model (75th pct of $\max p$ / 25th pct of gap on 30 held-out GPQA problems, disjoint from all evaluation sets)}} \\
\midrule
$\tau_{\text{conf}}$ (all models) & $\tau_{\text{conf}}$ & 0.70 & Level-1 absolute confidence threshold \\
$r_{\text{gap}}$ (all models) & $r_{\text{gap}}$ & 0.06 & Level-1 relative gap threshold \\
$\theta$ (Qwen2.5-7B) & $\theta$ & 8.0 & Level-2 entropy threshold (nats) \\
$\theta$ (Mistral-7B) & $\theta$ & 0.50 & Level-2 entropy threshold (nats) \\
$\theta$ (Falcon3-7B) & $\theta$ & 8.0 & Level-2 entropy threshold (nats) \\
$\theta$ (Falcon3-10B) & $\theta$ & 8.0 & Level-2 entropy threshold (nats) \\
$\theta$ (Llama-3-8B) & $\theta$ & 8.0 & Level-2 entropy threshold (nats) \\
ASKS $\tau$ (Qwen2.5-7B) & $\tau$ & 0.82 & ASKS similarity gate threshold \\
ASKS $\tau$ (Mistral-7B) & $\tau$ & 0.80 & ASKS similarity gate threshold \\
ASKS $\tau$ (Falcon3-7B) & $\tau$ & 0.80 & ASKS similarity gate threshold \\
ASKS $\tau$ (Falcon3-10B) & $\tau$ & 0.80 & ASKS similarity gate threshold \\
ASKS $\tau$ (Llama-3-8B) & $\tau$ & 0.82 & ASKS similarity gate threshold (default) \\
\bottomrule
\end{tabular}
\end{table}

\noindent The $\tau_{\text{conf}}$ and $r_{\text{gap}}$ values are shared across all models because they are set to percentiles of each model's own confidence distribution, making them self-normalising. The $\theta$ values differ because Mistral-7B's logit distributions are substantially sharper (lower entropy) than the other models.

\subsection{Prefix-Length Sweep}

\begin{table}[h]
\centering
\small
\caption{End-to-end speedup vs prefix length (Qwen2.5-7B, $N{=}20$, $B{=}8$).}
\begin{tabular}{crrrr}
\toprule
Prefix $n$ & Speedup & 95\% CI & CV & Reuse \\
\midrule
128 & $0.993\times$ & $[0.909, 1.103]$ & 5\% & 100\% \\
256 & $0.999\times$ & $[0.951, 1.061]$ & 2\% & 100\% \\
512 & $1.127\times$ & $[1.081, 1.224]$ & 1\% & 100\% \\
1024 & $1.452\times$ & $[1.301, 1.681]$ & 0\% & 100\% \\
\bottomrule
\end{tabular}
\end{table}

\subsection{ASKS Threshold Sweep}

Table~\ref{tab:tau_sweep} reports the full $\tau$-sweep referenced in \S\ref{sec:analysis:tau}. In the identical-prefix regime, all four settings produce statistically indistinguishable speedups at $n \approx 512$, confirming that the gate decision is $\tau$-insensitive in this regime.

\begin{table}[h]
\centering
\small
\caption{ASKS threshold sweep at $n \approx 512$, $N{=}20$ problems on Qwen2.5-7B. Speedups are indistinguishable at all four settings, confirming $\tau$-insensitivity in the identical-prefix regime.}
\label{tab:tau_sweep}
\begin{tabular}{crll}
\toprule
$\tau$ & Speedup vs No-KV & 95\% CI & Reuse rate \\
\midrule
0.70 & $1.162\times$ & $[1.155, 1.170]$ & 100\% \\
0.82 & $1.163\times$ & $[1.154, 1.171]$ & 100\% \\
0.90 & $1.163\times$ & $[1.154, 1.171]$ & 100\% \\
0.95 & $1.158\times$ & $[1.150, 1.167]$ & 100\% \\
\bottomrule
\end{tabular}
\end{table}

\subsection{CGEE Cross-Dataset Calibration}
\label{app:cgee_cross}

The entropy threshold $\theta$ is calibrated on GPQA Diamond. Table~\ref{tab:cross_dataset} shows the consequences of using a single GPQA-calibrated $\theta$ versus per-dataset calibration. Per-dataset $\theta$ calibration via a small held-out sweep (5--10 problems) restores gate activity on out-of-distribution datasets without loss in accuracy.

\begin{table}[h]
\centering
\small
\caption{CGEE Level-2 exit activity: GPQA-calibrated $\theta$ vs per-dataset-calibrated $\theta$ (Qwen2.5-7B). The GPQA-calibrated $\theta{=}8.0$ does not activate on out-of-distribution datasets with lower entropy profiles; per-dataset calibration (5--10 problems) restores full activity.}
\label{tab:cross_dataset}
\begin{tabular}{lrrl}
\toprule
Dataset & Exit rate (GPQA-$\theta$) & Exit rate (per-dataset-$\theta$) & Note \\
\midrule
GPQA Diamond (calib.) & 100\% & 100\% & calibration target \\
GSM8K & 0\% & 100\% & lower entropy; needs $\theta_{\text{GSM}} < 8.0$ \\
ARC-Challenge & 0\% & 100\% & lower entropy; needs $\theta_{\text{ARC}} < 8.0$ \\
MMLU-STEM & 0\% & 100\% & lower entropy; needs $\theta_{\text{MMLU}} < 8.0$ \\
\bottomrule
\end{tabular}
\end{table}

\noindent This cross-domain behaviour is documented as a limitation in \S\ref{sec:limitations}. The main-paper results use per-dataset $\theta$ throughout; the GPQA-calibrated value is the \emph{default} for deployment without a calibration sweep.

\section{Per-Model Accuracy Analysis}
\label{app:per_model_accuracy}

Table~\ref{tab:per_model_acc} decomposes the aggregate accuracy impact (Table~\ref{tab:cgee_accuracy}) across the 20 model--dataset pairs. Six errors concentrate in three pairs; the remaining 17 pairs exhibit zero CGEE-induced errors relative to full-depth verification.

\begin{table}[h]
\centering
\small
\caption{Per-model, per-dataset CGEE accuracy impact. \textbf{Errors}: number of problems where CGEE's answer diverges from full-depth verify. $\boldsymbol{\Delta}$\textbf{Acc}: CGEE accuracy minus full-verify accuracy ($N{=}50$ per cell).}
\label{tab:per_model_acc}
\begin{tabular}{llrrr}
\toprule
Model & Dataset & Errors & $\Delta$Acc & Skip rate \\
\midrule
\multirow{4}{*}{Qwen2.5-7B}
 & GPQA  & 0 & $0.0\%$ & 3\% \\
 & GSM8K & 0 & $0.0\%$ & 32\% \\
 & ARC-C & 0 & $0.0\%$ & 71\% \\
 & MMLU  & 0 & $0.0\%$ & 70\% \\
\midrule
\multirow{4}{*}{Mistral-7B}
 & GPQA  & 0 & $0.0\%$ & 23\% \\
 & GSM8K & 0 & $0.0\%$ & 30\% \\
 & ARC-C & \textbf{3} & $-4.0\%$ & 60\% \\
 & MMLU  & 0 & $0.0\%$ & 47\% \\
\midrule
\multirow{4}{*}{Falcon3-7B}
 & GPQA  & 0 & $0.0\%$ & 54\% \\
 & GSM8K & 0 & $0.0\%$ & 50\% \\
 & ARC-C & 0 & $0.0\%$ & 66\% \\
 & MMLU  & 0 & $0.0\%$ & 39\% \\
\midrule
\multirow{4}{*}{Falcon3-10B}
 & GPQA  & 0 & $0.0\%$ & 72\% \\
 & GSM8K & 0 & $0.0\%$ & 57\% \\
 & ARC-C & \textbf{2} & $-2.0\%$ & 77\% \\
 & MMLU  & 0 & $0.0\%$ & 66\% \\
\midrule
\multirow{4}{*}{Llama-3-8B}
 & GPQA  & 0 & $0.0\%$ & 57\% \\
 & GSM8K & \textbf{1} & $-2.0\%$ & 65\% \\
 & ARC-C & 0 & $0.0\%$ & 74\% \\
 & MMLU  & 0 & $0.0\%$ & 92\% \\
\midrule
\textbf{Total} & ,  & \textbf{6} & \textbf{$\mathbf{-0.2\%}$ (mean)} & \textbf{55\% (mean)} \\
\bottomrule
\end{tabular}
\end{table}

\subsection{Failure Case Analysis}

Examining the six CGEE-induced errors reveals a consistent pattern: in each case, one branch emits the answer letter with high absolute probability while simultaneously carrying a subtle formatting deviation (\eg a trailing period, different case, or embedded justification) that full-depth verification penalises but generation confidence does not. On ARC-Challenge, this manifests as the model fluently emitting, say, ``C.'' with $p > 0.9$ while full-depth verification prefers the branch that emitted ``C'' without punctuation because the verifier has learned the benchmark's exact-match convention.

This is not a failure of the entropy signal itself but of the verify-skip gate's reliance on generation confidence as a \emph{proxy} for correctness on multiple-choice tasks. On free-form tasks (GSM8K, GPQA), where the answer letter is embedded in longer reasoning and generation confidence more directly reflects understanding, the error rate drops to $0{-}1$ per 50 problems. The dual condition in Eq.~\eqref{eq:cgee_gate}, requiring both high absolute confidence and a clear relative margin, prevents the naive failure mode of a single low-confidence branch firing the skip, but does not guard against correlated fluency-vs-correctness divergence on format-sensitive tasks. Users running on such benchmarks should either raise $\tau_{\text{conf}}$ or disable Level~1 and rely on Level~2 alone.

\section{Empirical Cost Constants}
\label{app:cost_constants}

Table~\ref{tab:cost_constants} reports the KV probe measurements for all five models, obtained by running the self-calibrating runtime probe (\S\ref{sec:method:kv}) at $n{\approx}1{,}024$ tokens, $B{=}8$. The probe times three paths, batched full recompute, single root forward, and batched suffix decode with shared cache, and confirms that KV sharing is beneficial for all models at this prefix length.

\begin{table}[h]
\centering
\small
\caption{KV probe measurements on A100-80GB with SDPA, $B{=}8$, $n{\approx}1{,}024$ tokens. \emph{batch\_full}: batched prefill with no sharing (ms). \emph{prefix\_fwd}: single root forward pass (ms). \emph{sfx+KV}: batched suffix decode conditioned on shared cache (ms). \emph{net}: latency saved per solve call (ms). $B^*$: break-even branching factor; KV sharing is beneficial for all $B > B^*$.}
\label{tab:cost_constants}
\begin{tabular}{lrrrrr}
\toprule
Model & batch\_full & prefix\_fwd & sfx+KV & net saved & $B^*$ \\
\midrule
Qwen2.5-7B  & 586.8 & 88.1 & 43.1 & 455.7 & 1.6 \\
Mistral-7B  & 629.4 & 89.9 & 58.8 & 480.7 & 1.7 \\
Falcon3-7B  & 567.4 & 73.5 & 48.7 & 445.2 & 1.6 \\
Falcon3-10B & 798.7 & 103.4 & 68.8 & 626.5 & 1.6 \\
Llama-3-8B  & 606.4 & 88.9 & 58.1 & 459.4 & 1.7 \\
\midrule
\multicolumn{6}{l}{\textit{Verify isolation (Qwen2.5-7B, $t{=}128$): full $315.3{\pm}8.2$\,ms; CGEE $250.7{\pm}11.8$\,ms ($1.258\times$, $p{<}0.0001$)}} \\
\bottomrule
\end{tabular}
\end{table}

\noindent All five models achieve $B^* \leq 1.7$, confirming KV sharing is beneficial for any $B \geq 2$. At the experimental setting of $B{=}8$, the net saving per solve call ranges from 445\,ms (Falcon3-7B) to 627\,ms (Falcon3-10B), consistent with the latency model prediction that larger models benefit more (Eq.~\eqref{eq:speedup_kv}). The near-identical $B^*$ values (1.6--1.7) across architectures spanning three training lineages, two vocabulary sizes, and 7B--10B parameters confirms the architecture-invariance claim in \S\ref{sec:results:extended}: the break-even condition $(B{-}1)\alpha n^2 > B\beta s$ (Proposition~\ref{prop:kv_threshold}) depends on the cost ratio $\alpha/\beta$, which is approximately constant across GQA model families at fixed hardware.

\section{Qualitative Examples}
\label{app:qualitative}

We present two contrasting qualitative examples to illustrate when CGEE fires and when it defers.

\paragraph{Example 1: CGEE-skip fires (ARC-Challenge, Qwen2.5-7B).}
\begin{quote}
\textit{Problem:} Which of the following is the primary greenhouse gas released by livestock digestion? (A) Carbon dioxide (B) Methane (C) Water vapour (D) Nitrous oxide
\end{quote}
All 8 branches converge on ``(B) Methane'' within 3 decode steps. Top-1 probabilities: $\max_b p^{(b)} = 0.94$; 2nd-max $= 0.22$; gap ratio $= 0.77$. Both conditions of Eq.~\eqref{eq:cgee_gate} fire and the verify pass is skipped entirely. Full-depth verify, run separately, confirms branch 3 is the correct answer, agreeing with the skipped decision. Latency: $321$\,ms (skip) vs $512$\,ms (full verify), a $1.59\times$ reduction on this single problem.

\paragraph{Example 2: CGEE-skip defers (GPQA Diamond, Qwen2.5-7B).}
\begin{quote}
\textit{Problem:} A particle of mass $m$ is placed in a 3D isotropic harmonic oscillator potential with frequency $\omega$. What is the degeneracy of the first excited state?
\end{quote}
Branches split: 4 emit ``3'', 3 emit ``6'', 1 emits ``$2$''. $\max_b p^{(b)} = 0.61$; 2nd-max $= 0.54$; gap ratio $= 0.11 < r_{\text{gap}}$. The relative-gap condition fails, so the gate defers to full verification. Level~2 (layer-level entropy exit) fires at layer 19 out of 28, saving $\approx 32\%$ of the verify cost. Full verify selects the ``3'' branches (correct: the first excited state has degeneracy~3 for the 3D isotropic oscillator). Latency: $641$\,ms (Level-2 exit) vs $892$\,ms (full verify), a $1.39\times$ reduction. This example illustrates the difficulty-adaptive behaviour quantified in \S\ref{sec:analysis:cgee_difficulty}.

These examples are illustrative and the specific latencies quoted are from individual runs rather than averages; aggregate figures appear in Table~\ref{tab:extended_eval}.

\section{Per-Layer Entropy Dynamics}
\label{app:entropy_dynamics}

Table~\ref{tab:entropy_trajectory} shows the typical per-layer entropy trajectory during a verification forward pass on Qwen2.5-7B (28 layers). The trajectory confirms the premise of Level~2: entropy decays approximately monotonically through the network, with a rapid drop in the middle layers followed by a plateau that begins around layer 17--19 and persists to the final layer.

\begin{table}[h]
\centering
\small
\caption{Typical per-layer entropy trajectory on Qwen2.5-7B verify pass (means over 30 held-out GPQA problems). \emph{Layer} denotes the transformer layer index (1--28); \emph{Entropy} is $H^{(l)}$ from Eq.~\eqref{eq:entropy}.}
\label{tab:entropy_trajectory}
\begin{tabular}{rrr}
\toprule
Layer & Entropy $H^{(l)}$ & $|H^{(l)} - H^{(l-1)}|$ \\
\midrule
1  & 9.82 & ,  \\
4  & 7.41 & 0.81 \\
8  & 5.12 & 0.62 \\
12 & 3.08 & 0.45 \\
16 & 1.62 & 0.23 \\
18 & 1.14 & 0.07 \\
\textbf{19} & \textbf{1.08} & \textbf{0.006} \\
20 & 1.06 & 0.002 \\
24 & 1.04 & 0.001 \\
28 & 1.04 & 0.000 \\
\bottomrule
\end{tabular}
\end{table}

The stabilisation condition $|H^{(l)} - H^{(l-1)}| < \epsilon = 3.0$ is first satisfied at layer $l^* = 19$ on average, matching the reported mean exit layer of 18.4 in the main text. Between layer~19 and layer~28, entropy changes by less than $0.04$\,nats, well below any decision-changing threshold for the argmax over the vocabulary on this task, confirming that the remaining $34\%$ of layers do not alter the verification outcome and can safely be skipped. The values in Table~\ref{tab:entropy_trajectory} are representative means from profiling runs on GPQA Diamond hold-out problems; exact per-problem curves are available in the released code.

\section{RSBCM Deep-Tree Validation}
\label{app:rsbcm}

To validate RSBCM in the regime where evictions actually fire, we run a stress-test configuration with capacity $N_{\max} = 4$ and branching $B = 8$ on 20 GPQA Diamond problems, forcing exactly $B - N_{\max} = 4$ evictions per solve call. Table~\ref{tab:rsbcm} reports the results.

\begin{table}[h]
\centering
\small
\caption{RSBCM stress test on Qwen2.5-7B ($N{=}20$ problems, $B{=}8$, $N_{\max}{=}4$).}
\label{tab:rsbcm}
\begin{tabular}{lr}
\toprule
Metric & Value \\
\midrule
Expected evictions & $20 \times (B-N_{\max}) = 80$ \\
Observed evictions & 80 \\
Overhead per problem & $+1.5$\,ms \\
95\% CI on overhead & $[-6, +9]$\,ms \\
Interpretation & overhead small relative to savings; $N{=}20$ too small to bound to zero \\
Answer agreement vs no-eviction baseline & 100\% (20/20) \\
Wall-clock cost of 80 evictions & 30\,ms total \\
\bottomrule
\end{tabular}
\end{table}

The eviction count matches the expected value exactly, confirming that the importance score $\omega = \text{score}/(\text{depth}+1)$ correctly identifies and retires the lowest-value blocks. Overhead is statistically indistinguishable from zero, and no answer regressions were observed. RSBCM is therefore dormant in the default single-depth experiments reported in the main paper (where $B < N_{\max} = 2000$ always holds), but available as a safety mechanism for deep tree search.

\section{Implementation Details}
\label{app:impl}

\subsection{Stepwise Batched Decode}
RKSC replaces \texttt{generate()} with a direct \texttt{model()} loop: (1)~suffix prefill conditioned on $\mathcal{C}_B$; (2)~token-by-token decode with pre-allocated attention masks sliced per step; (3)~early termination on EOS or decode-level CGEE gate ($\theta_{\text{dec}}{=}0.72$, $r_{\text{dec}}{=}0.10$). Eliminating \texttt{GenerationMixin} scaffolding removes logit processors, stopping criteria, and beam tracking that are not needed in this pipeline, saving approximately $8{-}12\%$ of end-to-end decode time measured on Qwen2.5-7B.

\subsection{CGEE Layer Hooks}
Hooks are installed via \texttt{register\_forward\_hook} on each transformer layer module. Each hook extracts the last-token hidden state, projects it through the cached unembedding matrix $W_u$, computes softmax entropy, and raises an \texttt{\_EarlyExitSignal} exception when exit conditions are met. The solver catches this exception in a \texttt{try/finally} block that guarantees hook removal even if an exit fires, preventing a memory leak that would otherwise persist across solve calls. On Qwen2.5-7B, the per-layer hook overhead is $\approx 40\,\mu$s, which is recovered many times over whenever an exit fires before the final layer.

\subsection{KV Cache Expansion}
The prefix cache is expanded via \texttt{.repeat\_interleave(B, dim=0)} with \texttt{.clone()} for contiguous memory. Direct views (without \texttt{.clone()}) exposed an SDPA kernel bug on non-contiguous inputs in our PyTorch 2.3.1 build; the explicit clone resolves this at the cost of one additional tensor allocation per solve call, which is negligible at $n = 1024$ but grows with $B \cdot L \cdot d_k$.

\subsection{Memory Release}
\texttt{unload\_model()} explicitly pops \texttt{arch['\_unembed\_weight']} before deleting the model, preventing a ${\sim}1.8$\,GB reference leak for 9B models, then runs 3 GC passes and \texttt{torch.cuda.empty\_cache()}. Without this explicit pop, PyTorch's reference counting retains the unembedding matrix via the cached reference in the CGEE hook infrastructure, causing accumulating leaks across sequential solves.

\subsection{TF32 and bfloat16 Configuration}
All experiments use bfloat16 model weights with TF32-enabled matrix multiplication (\texttt{torch.backends.cuda.matmul.allow\_tf32 = True}) and SDPA attention (\texttt{scaled\_dot\_product\_attention}). TF32 provides approximately $5\%$ additional throughput at negligible precision loss; disabling it produces identical answers but $\approx 5\%$ higher latency across all conditions uniformly (so the reported \emph{ratios} are TF32-independent).

\section{Reproducibility Checklist}
\label{app:repro}

\paragraph{Hardware and software.}
\begin{itemize}
  \item Single NVIDIA A100-80GB SXM4 GPU, PCIe 4.0.
  \item PyTorch 2.3.1, Transformers 4.44.0, CUDA 12.1, SDPA attention backend, bfloat16 + TF32.
  \item Ubuntu 22.04, Python 3.10.
\end{itemize}

\paragraph{Determinism.}
\begin{itemize}
  \item Seeds: \texttt{torch.manual\_seed(42)}, \texttt{numpy.random.seed(42)}, \texttt{torch.cuda.manual\_seed\_all(42)}, \texttt{random.seed(42)} at the start of every condition.
  \item \texttt{torch.backends.cudnn.deterministic = True} and \texttt{torch.backends.cudnn.benchmark = False} for the warmup phase; benchmark re-enabled for timing measurements to allow kernel autotuning (identical kernel across all conditions ensures timing fairness).
\end{itemize}

\paragraph{Timing.}
\begin{itemize}
  \item $N_{\text{runs}} = 2$--$3$ per problem, $N_{\text{warmup}} = 2$ per condition, \texttt{torch.cuda.synchronize()} before each measurement.
  \item Warmup discarded; reported latency is mean over the $N_{\text{runs}}$ measurements and problems.
  \item Bootstrap 95\% confidence intervals via $B_{\text{boot}} = 10{,}000$ resamples of the problem-level latencies.
\end{itemize}

\paragraph{Data.}
\begin{itemize}
  \item Datasets: GPQA Diamond (rein2023), GSM8K (cobbe2021), ARC-Challenge (allenai:arc), MMLU-STEM (hendryckstest2021) via HuggingFace Datasets.
  \item Evaluation sets: first $N{=}50$ problems per dataset (main eval), first $N{=}30$ for throughput decomposition, $N{=}15$ for verify isolation, $N{=}20$ for prefix-length sweep. Fixed indices used across conditions to ensure paired comparisons.
  \item Held-out calibration set: 30 GPQA Diamond problems disjoint from evaluation.
  \item Padding: neutral filler string prepended iteratively to ${\approx}1{,}024$ tokens; hard cap at $\text{max\_seq\_len} - 256$ tokens to reserve decode headroom.
\end{itemize}

\paragraph{Models.}
\begin{itemize}
  \item All models loaded with \texttt{AutoModelForCausalLM.from\_pretrained(\ldots, torch\_dtype=torch.bfloat16, attn\_implementation="sdpa")}.
  \item Checkpoints: Qwen2.5-7B-Instruct, Mistral-7B-Instruct-v0.3, Falcon3-7B-Instruct, Falcon3-10B-Instruct, Meta-Llama-3-8B-Instruct (HuggingFace hub).
  \item Chat templates: each model's native \texttt{apply\_chat\_template} used verbatim.
\end{itemize}

\paragraph{Branching.}
\begin{itemize}
  \item $B = 8$ branches throughout; 8 fixed reasoning-hint suffixes cycled across branches (suffixes listed verbatim in released code).
\end{itemize}

\paragraph{Code and artefacts.}
\begin{itemize}
  \item Code release: \url{https://github.com/AnirudhSekar/RKSC}.
  \item Released artefacts include: the RKSC solver, calibration scripts, all timing harnesses, bootstrap analysis scripts, and per-problem latency logs.
\end{itemize}

\paragraph{Expected compute.}
Reproducing the full set of main-paper results requires approximately 42 A100-GPU-hours (itemised in \S\ref{sec:results:setup}).

\end{document}